\documentclass[conference]{IEEE-conference-template/IEEEtran}
\IEEEoverridecommandlockouts

\usepackage{cite}
\usepackage{amsmath,amssymb,amsfonts}
\usepackage{algorithmic}
\usepackage{graphicx}
\usepackage{textcomp}
\usepackage{xcolor}
\def\BibTeX{{\rm B\kern-.05em{\sc i\kern-.025em b}\kern-.08em
    T\kern-.1667em\lower.7ex\hbox{E}\kern-.125emX}}
\input{definitions}
    
\begin{document}

\title{
    \textit{Wafer Defect Root Cause Analysis with Partial Trajectory Regression}\\
    {\huge DM: Big Data Management and Machine Learning}
}

\author{\IEEEauthorblockN{Kohei Miyaguchi$^*$\thanks{$^*$ Kohei Miyaguchi is currently affiliated with LY Research, Japan.}}
\IEEEauthorblockA{\textit{IBM Research -- Tokyo} \\
Tokyo, Japan \\
koheimiyaguchi@gmail.com}
\and
\IEEEauthorblockN{Masao Joko}
\IEEEauthorblockA{\textit{IBM Semiconductors} \\
Tokyo, Japan \\
mjoko@jp.ibm.com}
\and
\IEEEauthorblockN{Rebekah Sheraw}
\IEEEauthorblockA{\textit{IBM Semiconductors} \\
Albany, NY, USA \\
rebekah.sheraw@ibm.com}
\and
\IEEEauthorblockN{Tsuyoshi Id\'e}
\IEEEauthorblockA{\textit{IBM Semiconductors} \\
Yorktown Heights, NY, USA \\
tide@us.ibm.com}
}

\maketitle

\begin{abstract}
Identifying upstream processes responsible for wafer defects is challenging due to the combinatorial nature of process flows and the inherent variability in processing routes, which arises from factors such as rework operations and random process waiting times. 
This paper presents a novel framework for wafer defect root cause analysis, called Partial Trajectory Regression (PTR). The proposed framework is carefully designed to address the limitations of conventional vector-based regression models, particularly in handling variable-length processing routes that span a large number of heterogeneous physical processes. 
To compute the attribution score of each process given a detected high defect density on a specific wafer, we propose a new algorithm that compares two counterfactual outcomes derived from partial process trajectories. This is enabled by new representation learning methods, proc2vec and route2vec. 
We demonstrate the effectiveness of the proposed framework using real wafer history data from the NY CREATES fab in Albany.
\end{abstract}

\section{Introduction}

Identifying the root causes of wafer defects is one of the most critical yet challenging tasks in semiconductor manufacturing. Several unique factors contribute to this complexity, including the intricate dependencies between processes, the heterogeneity of individual operations and recipes, and the dynamic variations in processing routes.

Particularly during process integration and yield ramp-up stages, classical design-of-experiment methodologies, which analyze process outcomes across systematically varied parameters, are often impractical due to too many adjustable parameters. Moreover, although a wide range of off-the-shelf machine learning (ML) tools are publicly available, most of these tools are designed for prediction tasks, such as estimating real-valued outputs (regression) or categorical outcomes (classification). As a result, these tools are poorly suited to the task of root cause analysis, and fab-wide defect diagnosis still heavily depends on manual, ad hoc analysis by domain experts.

\begin{figure}[t]
    \centering
    \includegraphics[clip, trim=6cm 6cm 8cm 4cm, width=\linewidth]{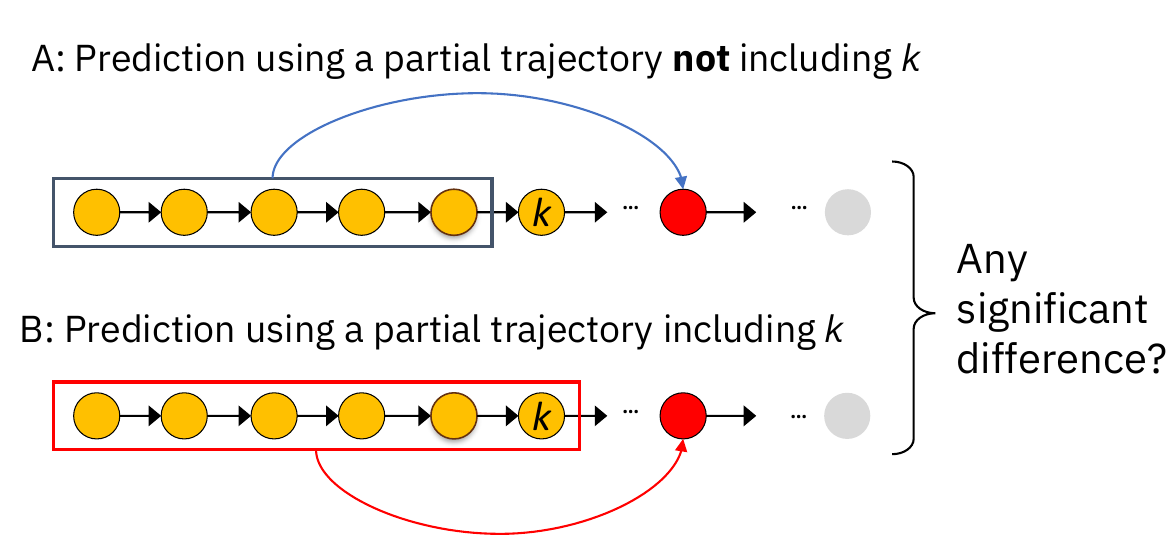}
    \caption{Causal analysis framework with partial trajectory regression (PTR).}
    \label{fig:PTR_concept}
\end{figure}

To enable more systematic wafer defect root-cause analysis, two major directions have been pursued in the literature in semiconductor analytics to date. The first approach treats defect diagnosis as a by-product of cross-process virtual metrology (VM) modeling. Regularized linear regression combined with variable selection techniques is commonly used (e.g.,~\cite{susto2015multi,jebri2016virtual,kim2018variable}). However, linear models struggle to capture complex nonlinear relationships across heterogeneous fabrication processes. Furthermore, defect attribution based on linear models is essentially reduced to variable-wise correlation analysis, which is known to yield only weak attribution signals~\cite{Miyaguchi25ASMC}. While recurrent neural networks (RNNs) and Transformers can capture complex nonlinear dependencies in sequential processes (e.g.,~\cite{yella2021soft,han2023deep,dalla2023deep,lee2020recurrent,hsu2023virtual}), they operate as black boxes, making input attribution a non-trivial task.

The second approach involves leveraging explainable artificial intelligence (XAI) techniques applicable to black-box prediction models. This may be a promising direction in that it can potentially enhance expressive prediction models with interpretability, helping to identify which process steps should be adjusted to improve defect rates. However, for cross-process defect attribution, existing methods such as those based on Shapley values (e.g.,~\cite{torres2020machine,senoner2022using,lee2023expandable,guo2024enhanced}) fail to reflect key characteristics of semiconductor processes. Limitations include a disregard for the sequential nature of fabrication and dependence on arbitrarily selected baseline inputs. Ironically, these methods attempt to explain black-box models by introducing another layer of black-box reasoning.

Those two technical directions highlight two main technical tasks in data-driven root-cause analysis. The first task is to predict process outcomes (e.g.,~defect density), given a sequence of physical processes. The second task is to provide an explanation to an observed pair of the process input and its outcome. We refer the first and second tasks the forward problem (prediction) and inverse problems (attribution), respectively. The main technical difficulty of wafer defect root-cause analysis is that these do not fit the standard problem settings of machine learning. 

Specifically, \textit{first},  process routes vary in length and tools used, making it difficult to represent inputs as fixed-dimensional vectors---a key assumption of many algorithms like random forests or gradient boosting. Related to this point, \textit{second}, each process step may involve completely different physical and chemical mechanisms, making it extremely difficult to extract features that are comparable across steps.  \textit{Third},  most off-the-shelf prediction and attribution tools ignore sequential structure; for example, linear regression and Shapley values assume unordered feature vectors.  \textit{Finally}, many XAI algorithms make assumptions that may not be appropriate in semiconductor manufacturing. Examples include the dependency on baseline inputs of the Shapley values.

To overcome these significant challenges, we propose a novel framework called the partially trajectory regression (PTR) for cross-process defect root-cause analysis. We formalize the prediction task as trajectory regression~\cite{ide2009travel,ide2011trajectory}, rather than conventional vector-based regression, to accommodate variable-length sequences. To handle process heterogeneity along processing routes, we introduce \texttt{proc2vec} and \texttt{route2vec}, two new representation learning approaches for encoding process trajectories, designed for effective learning even under limited sample sizes. Finally, as illustrated in Fig.~\ref{fig:PTR_concept}, we propose a novel attribution algorithm based on the new idea of comparing the outcomes between two \textit{partial} trajectories---one including the target process and the other excluding it.  We demonstrate its effectiveness on real wafer history data for the state-of-the-art FEOL (front end-of-line) process from the NY CREATES fab in Albany.

\section{Related Work}

There is a growing body of research on data-driven root cause analysis (RCA) methods in manufacturing. In a recent systematic review, Pietsch et al.~\cite{pietsch2024root} argue that traditional knowledge-based approaches have suffered from challenges in systematically acquiring domain knowledge. While data-driven approaches effectively mitigate the knowledge acquisition bottleneck, they still face methodological limitations, particularly in evaluating the causal relevance of variables in high-dimensional settings. This work aims to address this critical issue in data-driven RCA.

\subsection{Handling Process Heterogeneity}

A key characteristic that distinguishes semiconductor manufacturing from other industrial domains is its process complexity. A typical semiconductor process involves hundreds of intricate and highly specialized physical operations, including photolithography, thermal annealing, polishing, wet and dry etching, ion implantation, electroplating, sputtering, and chemical vapor deposition, among others.

For cross-process RCA, these heterogeneous operations must be mapped to a shared representation space to enable meaningful comparisons. Existing literature offers three general approaches for this. 
The first approach utilizes process trace data~\cite{xu2024fast,fan2022key}, such as temperature and pressure recorded from processing chambers. While effective within individual process tools, this method requires extensive tool-specific preprocessing, and the quality of analysis heavily depends on the chosen preprocessing strategy, making it less suitable for cross-process analysis. 
The second approach uses inline measurements as proxies for physical processes. Since these measurements partially absorb the physical heterogeneity across processes, this method has become common practice in recent studies~\cite{senoner2022using,guo2024enhanced,wang2024improved,ni2025novel}. However, these approaches typically disregard the sequential order of processes and perform RCA as a by-product of virtual metrology (VM), often via univariate correlation analysis, which is known to yield only weak attribution signals~\cite{Miyaguchi25ASMC}.

The third approach involves embedding techniques, where data objects (e.g., process steps) are transformed into numerical vector representations. This idea, widely adopted in natural language processing (NLP), such as in Word2Vec~\cite{mikolov2013efficient} or token embeddings in Transformer models~\cite{vaswani2017attention}, treats individual processes as tokens and process sequences as sentences. Embedding allows process similarity to be learned from data, enabling richer representations. 
Recently, embedding techniques have been introduced in the semiconductor context. For instance, Fan et al.~\cite{fan2022data} use one-hot encoding to unify categorical and numerical data in the VM setting. Schulz et al.~\cite{schulz2022graph} propose defining a fab state vector using known interdependencies among processing tools under an unsupervised setting, without the context of wafer defect analysis. 

The proposed \texttt{proc2vec} and \texttt{route2vec} fundamentally differ from those recent works in that they incorporate the similarity between related tools and recipes via the kernel function. This is critical in process integration stage since many similar but different recipes are tested with multiple tools and chambers of the same type. Naive one-hot encoding amounts to treating them as independent objects, which drastically deteriorates sample efficiency.

\subsection{Explainable AI (XAI)}

Numerous methods have been developed to improve the interpretability of machine learning models under the umbrella of explainable AI (XAI)~\cite{xu2019explainable}. One widely used category is additive explanation methods~\cite{lundberg2017unified,ribeiro2016why}, which provide mathematically justified decompositions of a model's output into individual contributions of input variables. Other common XAI techniques include gradient-based methods~\cite{selvaraju2020grad} and attention-based methods~\cite{ali2022xai}. 

Several XAI methods have been proposed for sequential input variables, such as wafer processing routes~\cite{rojat2021explainable}. The integration of model-agnostic XAI techniques with advanced quality prediction models has become a major research trend in semiconductor analytics, aiming to balance predictive power with interpretability. Among various XAI methods (see, e.g.,~\cite{molnar2020interpretable} for an overview), the Shapley value (SV) is the most widely used in semiconductor analytics~\cite{torres2020machine,senoner2022using,lee2023expandable,guo2024enhanced}.

Ironically, despite its widespread adoption, most studies apply XAI methods in a black-box fashion, without scrutinizing their modeling assumptions. In fact, SV disregards both the sequential nature of semiconductor processes and the wafer-lot hierarchy. Additionally, mainstream attribution algorithms such as the Shapley values, integrated gradients, and local linear surrogate models are deviation agnostic~\cite{ide2023generative}, meaning they do not account for how much a given input deviates from a normal or baseline process state, making them less suitable for defect diagnosis tasks.

\section{Problem Setting}

As discussed, there are two major machine learning tasks towards wafer defect root cause analysis. The first task is to learn a prediction model for the occurrence of a certain defect type, while the second task is to establish a scoring model for defect attribution, given the learned prediction model. 

\textbf{Prediction modeling}: 
We formalize the prediction task as trajectory regression for defect density. Specifically, we assume to have a training dataset including $N$ pairs of a defect density value $y$ and a process trajectory $\xi$: 
\begin{align}\label{eq:data}
    \mathcal{D}_\text{train}\coloneq \{(y^{(n)},\xi^{(n)})\}_{n=1}^N,
\end{align}
where the superscript $^{(n)}$ denotes that the quantity belongs to the $n$-th wafer. A process trajectory $\xi$ is a sequence of process attributes and timestamps as 
\begin{align}\label{eq:trajectory_as_(x,t)sequence}
    \xi=\left( (\bmx_1,t_1), \ldots, (\bmx_L,t_L) \right),
\end{align}
where $L$ is the number of process steps included in $\xi$. The goal of the prediction task is to find a function $f(\xi)$ such that it gives a prediction of defect density for a given process trajectory $\xi$ that is not included in $\calD_\text{train}$ in general. 

\textbf{Attribution modeling}: 
On the other hand, the attribution task is to find a function $\alpha_k(\xi)$, which computes the responsibility score for the $k$-th process on the process trajectory $\xi$, possibly making use of a pre-trained prediction function $f(\xi)$.

\textbf{Considerations in modeling}: 
Before delving into the technical details, we highlight the unique characteristics of semiconductor analytics. Although Eq.~\eqref{eq:data} might have given the impression that the dataset comprises independent samples, wafers are not generally independent due to lot- and batch-based processing. Additionally, in the process integration and yield ramp-up stage, multiple processing recipes are tested simultaneously. Hence, some trajectories can be quite similar except for a few processes. 

Although the nominal diversity of processing routes can be extremely large, the \textit{effective sample size}---defined as the number of wafers per unique processing or measurement condition---can be very small, potentially even as low as one. Nonetheless, the net variety of trajectories may be much smaller than their nominal diversity. Thus, leveraging similarities between tools and recipes is key to successful modeling. The \texttt{proc2vec} embedding approach described in the next section is designed to achieve that goal.

\section{Process Embedding}

The proposed partial trajectory regression (PTR) framework covers both prediction and attribution problems. The embedding module featuring \texttt{proc2vec} is used by the prediction and attribution function as a common building block. 

Now, let us take a look at how it works. The \texttt{proc2vec} algorithm generates a vector representation for each process using a kernel embedding technique. In our setting, we assume that high-level process attributes, such as process ID and recipe ID are available from the manufacturing execution system (MES) but detailed tool traces are unavailable at least in some tools. hence, for cross-process attribution, we use only high-level attributes as the first attempt. 

First, we pull common attributes out of the MES across different processes, such as equipment IDs, recipe IDs, tool types, photo layer IDs, route IDs, and others. We then create a single synthetic `token' for each process by concatenating those strings as:
\begin{align}\nonumber
    (\text{process token})=&\texttt{eqp}\oplus
    \texttt{recipe}\oplus
    \texttt{tool\_type}\oplus\\
    &\texttt{photo\_layer}\oplus
    \texttt{route}\oplus \cdots,
\end{align}
where $\oplus$ denotes concatenation with a suitable separator.

Considering the challenges of low effective sample size, we employ a relatively simple embedding approach. Specifically, based on the string representation, we first construct a dictionary of the tokens. Let $V_d$ is the size of the vocabulary, i.e.,~the number of unique tokens. For the dictionary, we compute the kernel matrix $\sfK \in \mathbb{R}^{V_d \times V_d}$, where $K_{i,j}$ is the similarity between token $i$ and token $j$.  

Our preliminary analysis indicated that a custom variant of the substring kernel~\cite{lodhi2002text,shawe2004kernel} effectively captured the clustering structure of processes. Once the kernel matrix $\sfK$ is computed, the vector representation of token $i$ is obtained as:
\begin{align}\label{eq:embedding_MDS}
    \bmx_i = (\sqrt{\lambda}_1v_i^{(1)}, \ldots, \sqrt{\lambda}_kv_i^{(k)},\ldots, \sqrt{\lambda}_D v_i^{(D)})^\top,
\end{align}
where $\lambda_k$ is the $k$-th largest eigenvalue of $\sfK$, and $v_i^{(k)}$ is the $i$-th element of the eigenvector corresponding to $\lambda_k$. $D$ is the embedding dimensionality as a hyperparameter.

\section{Trajectory Regression}

Now that processes have been converted into $D$-dimensional vectors, we consider developing a predictive model for defect density $y$ as a function of the sequence of process vectors. This remains a non-trivial task due to the variable length of trajectories, making it an instance of \textit{trajectory regression}, distinct from standard regression problems.

In trajectory regression modeling, a module called the \texttt{route2vec} plays the key role. It converts the entire trajectory, represented as an ordered set of \texttt{proc2vec} vectors, into another vector representation. We achieve this by introducing a constrained version of recurrent neural network (C-RNN), where each cell operates as:
\begin{align}
    \bmz_k = \text{Cell}(\bmz_{k-1},\bmx_k,t_k), 
\end{align}
where $\bmz_k$ is the vector representation of the \textit{partial trajectory} up to the $k$-th process from the first. Hence, the vector representation of the full trajectory of $\xi^{(n)}$ is obtained as $\bmz^{(n)} \coloneq \bmz_{L^{(n)}}^{(n)}$. We refer to this mapping from a (partial or full) trajectory to its corresponding vector representation $\bmz_k$ ($k=1,\ldots, L)$ as \texttt{route2vec}.

Our preliminary experiments showed that off-the-shelf RNN models did not yield promising results, as expected due to the small effective sample size. Instead, we propose a custom cell architecture:
\begin{align}\label{eq:cell_linear}
    \bmz_k =\psi(t_k,t_{k-1})\bmx_k + \bmz_{k-1}, 
\end{align}
where $\psi(\cdot,\cdot)$ is a temporal mapping function, set to $\log_{10}(1+\cdot)$ in our evaluation. Extending Eq.~\eqref{eq:cell_linear} to more expressive models such as multi-layer perceptrons (MLP) is straightforward; however, we leave further analysis for future work.

The regression module defines a function to predict $y$ from~$\xi$. Since \texttt{route2vec} converts trajectories of varying lengths into fixed-dimension vectors, the prediction function can be formulated as:
\begin{align}\label{eq:regression_MLP}
    f(\bmz_k)=\text{MLP}(\bmz_k),
\end{align}
where MLP denotes a multi-layer perceptron. We deliberately abused the notation for $f(\cdot)$ so that it admits the vector representation of not only a full trajectory but also partial trajectories. For the specific implementation, we used a fully connected network without hidden layers, defining a linear prediction function, in our empirical evaluation. The function $f(\cdot)$ is trained to minimize a properly regularized squared loss, such as
\begin{align}
    L\coloneq \frac{1}{2N}\sum_{n=1}^N\sum_{k=1}^{L^{(n)}}\frac{1}{L^{(n)}}\left(y^{(n)}- f(\bmz_k)\right)^2 + \nu \|\bmtheta\|_1,
\end{align}
where $\bmtheta$ is the model parameters in $f(\cdot)$ and $\|\cdot\|_1$ is the $\ell_1$-norm. $\nu$ is the regularization strength and is treated as a hyperparameter. The inner summation represents data augmentation for numerical stability. The details on training are left to a longer version of the paper in preparation.

\section{Attribution Using Partial Trajectory Regression}

One of the most important features of the PTR framework is its ability to admit \textit{partial} trajectories as input. To quantify the influence of the $k$-th process, we ask: Is there a significant difference in the prediction outcome between the partial trajectories $\bmz_k$ and $\bmz_{k-1}$? This comparison between two \textit{counterfactual} inputs measures the potential outcome when including process $k$.

Following Rubin's potential outcome framework~\cite{rubin2005causal}, the attribution module quantifies the \textit{causal intervention} of process $k$ as:
\begin{align}
    \alpha_k(\xi)=
        f(\bmz_k)-f(\bmz_{k-1}).
\end{align}
This defines the attribution score of the $k$-th process in $\xi$. The PTR attribution score satisfies the following additive property:

\begin{theorem}[Additive Property of PTR]
The PTR attribution score satisfies the additive property:
\begin{align}
    \sum_{l=1}^k\alpha_{l} = f(\bmz_k)-  f(\bmz_0).
\end{align}
\end{theorem}
\begin{proof}
    This follows directly from summing individual contributions:
\begin{align*}
\sum_{l=1}^k\alpha_{l} 
&= [f(\bmx_k)-  f(\bmz_{k-1})] + [f(\bmz_{k-1})-  f(\bmz_{k-2})] \\
&+ \cdots +[f(\bmz_1)-  f(\bmz_0)] = f(\bmz_k)- f(\bmz_0).
\end{align*}
\end{proof}
This additive property significantly enhances the interpretability of the attribution score, as demonstrated in the cumulative attribution plots in the next section.

Figure~\ref{fig:PTR_architecture} illustrates the overall architecture of the PTR framework. While we have presented a simple version of the model with the embedding (Eq.~\eqref{eq:embedding_MDS}) and the recurrent layer (Eq.~\eqref{eq:cell_linear}) fixed and `frozen' during the training, extending to a more general architecture is straightforward. Further analysis along this direction is left for future work. 

\begin{figure}[tb]
    \centering
    \includegraphics[clip, trim=6cm 2cm 6cm 2cm, width=\linewidth]{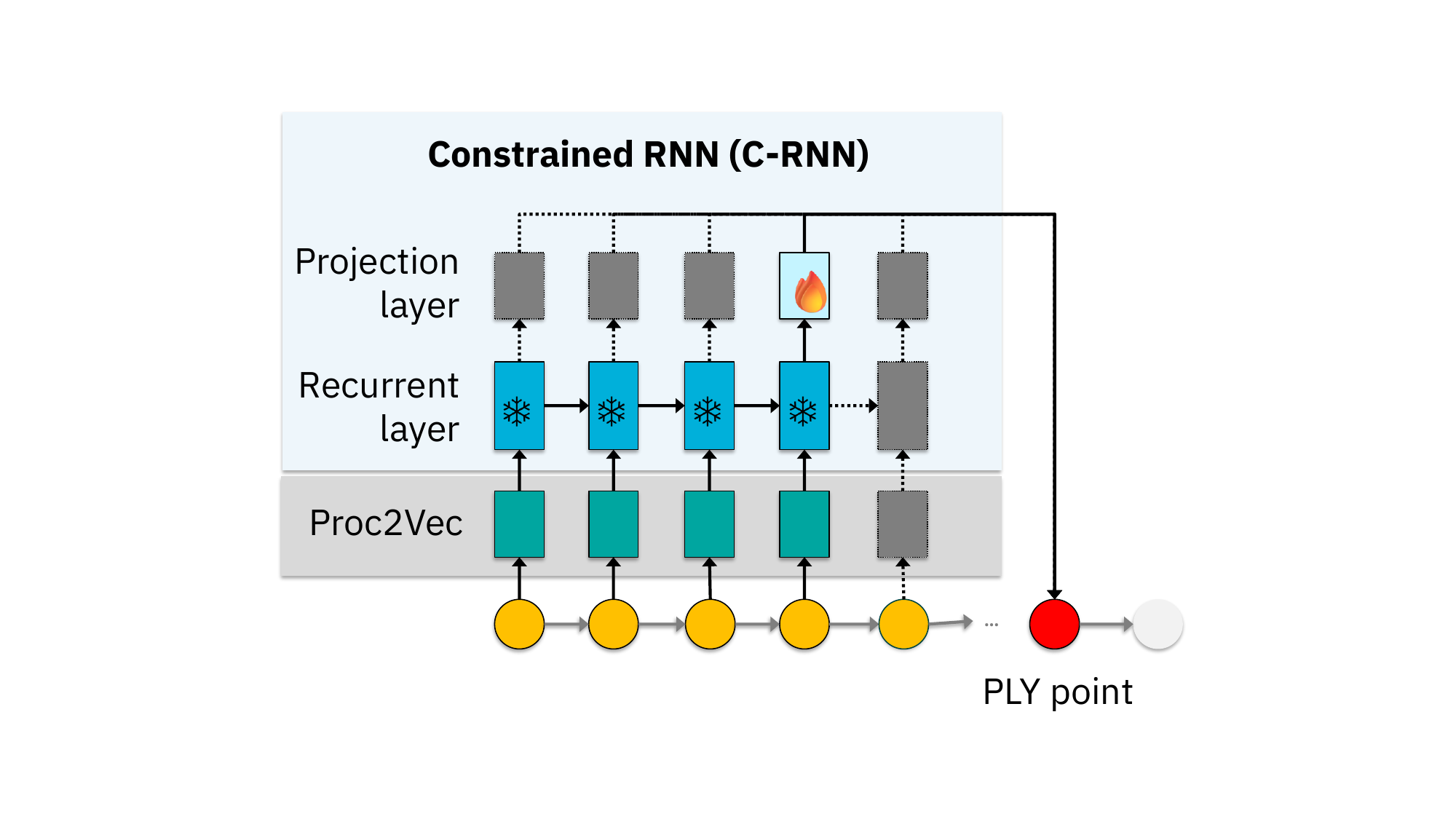}
    \caption{System architecture of the PTR framework. Only the projection layer is learnable in this architecture to ensure numerical stability under limited sample sizes.}
    \label{fig:PTR_architecture}
\end{figure}

\section{Empirical Evaluation}

We applied the PTR framework to conduct a root cause analysis (RCA) for a specific defect type in a state-of-the-art front-end-of-line (FEOL) process at the NY CREATES Albany NanoTech fab. The dataset consisted of process histories for $N=787$ wafers, covering hundreds of processes, along with corresponding defect density measurements obtained at a process-limited yield (PLY) evaluation point.

First, we validated the embedding module by comparing three scenarios: 1) $\bmx_i = 1$ (a constant value), 2) $\bmx_i$ represented as a one-hot vector, and 3) the proposed kernel embedding approach. The correlation coefficients with defect density $y$ were 0.27, 0.52, and 0.61, respectively, demonstrating the effectiveness of the kernel embedding technique. 

Figure~\ref{fig:MDS} shows the distribution of the resulting embeddings, where distances between processes reflect the similarity of tools and recipes. The visualization was performed with \texttt{scikit-learn}'s implementation of $t$-SNE ($t$-distributed stochastic neighbor embedding)~\cite{van2008visualizing}. The figure exhibits a clear clustering structure according to equipment types, which suggests that the string kernel approach effectively captures inter-process similarities. 

\begin{figure}
    \centering
    \includegraphics[clip, trim=0cm 3.cm 0cm 3.cm, width=\linewidth]{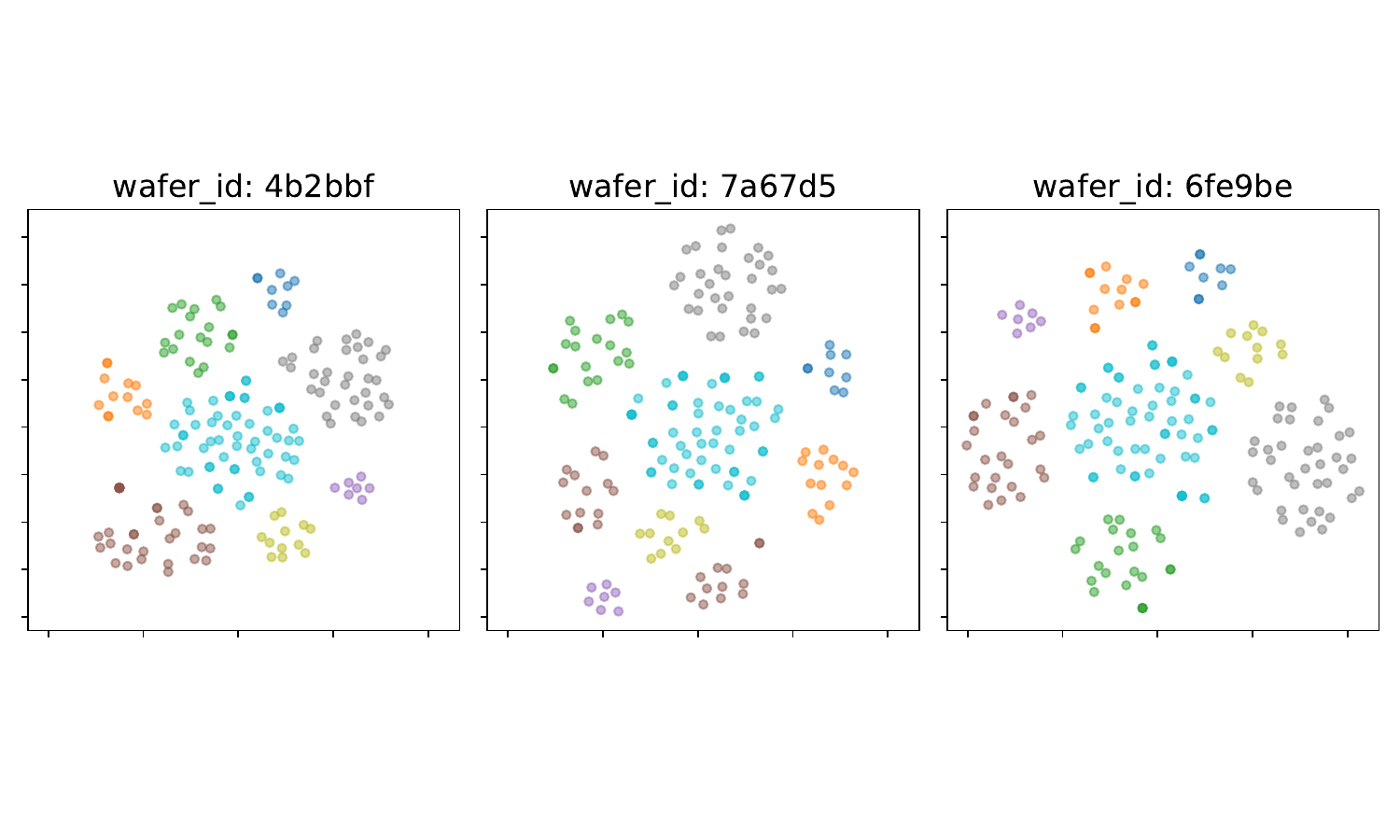}
    \caption{Distribution of the \texttt{proc2vec} embeddings for three wafer samples, visualized with $t$-SNE (the perplexity is set to be 30). The color coding corresponds to distinction among the equipment types including Wet Process, Rapid Thermal Processing, Inspection, Lithography, Reactive Ion Etching, Ion Implantation, Furnace, Chemical Mechanical Polishing, although the exact mapping as well as axis scales are undisclosed.}
    \label{fig:MDS}
\end{figure}

Next, we evaluated the regression module. While the linear model used in Eq.~\eqref{eq:regression_MLP} provided moderate predictive performance, the correlation coefficient improved to 0.87 when incorporating additional process features and optimizing the prediction model. Further details are omitted here for brevity.

Finally, Fig.~\ref{fig:CT_scoring} presents the cumulative attribution score, which plots 
\begin{align}
    \sum_{i=1}^k \alpha_i + f(\bmz_0)
\end{align} 
for a specific wafer from the held-out dataset at each timestamp $\tau$, where $t_k \leq \tau$. The plot illustrates the progression from $f(\bmz_0)$ (intercept) to $f(\bmz)$, the final prediction based on the wafer’s complete trajectory. Notable jumps, labeled A and B, are highlighted in the plot. Upon further inspection, these jumps correspond to unusually long waiting times at certain tools, suggesting potential root causes of high defect density. This example demonstrates how PTR effectively identifies problematic processes, providing actionable insights for RCA.

\begin{figure}
    \centering
    \includegraphics[clip, trim=2.1cm 1.1cm 2.cm 2cm, width=0.48\textwidth]{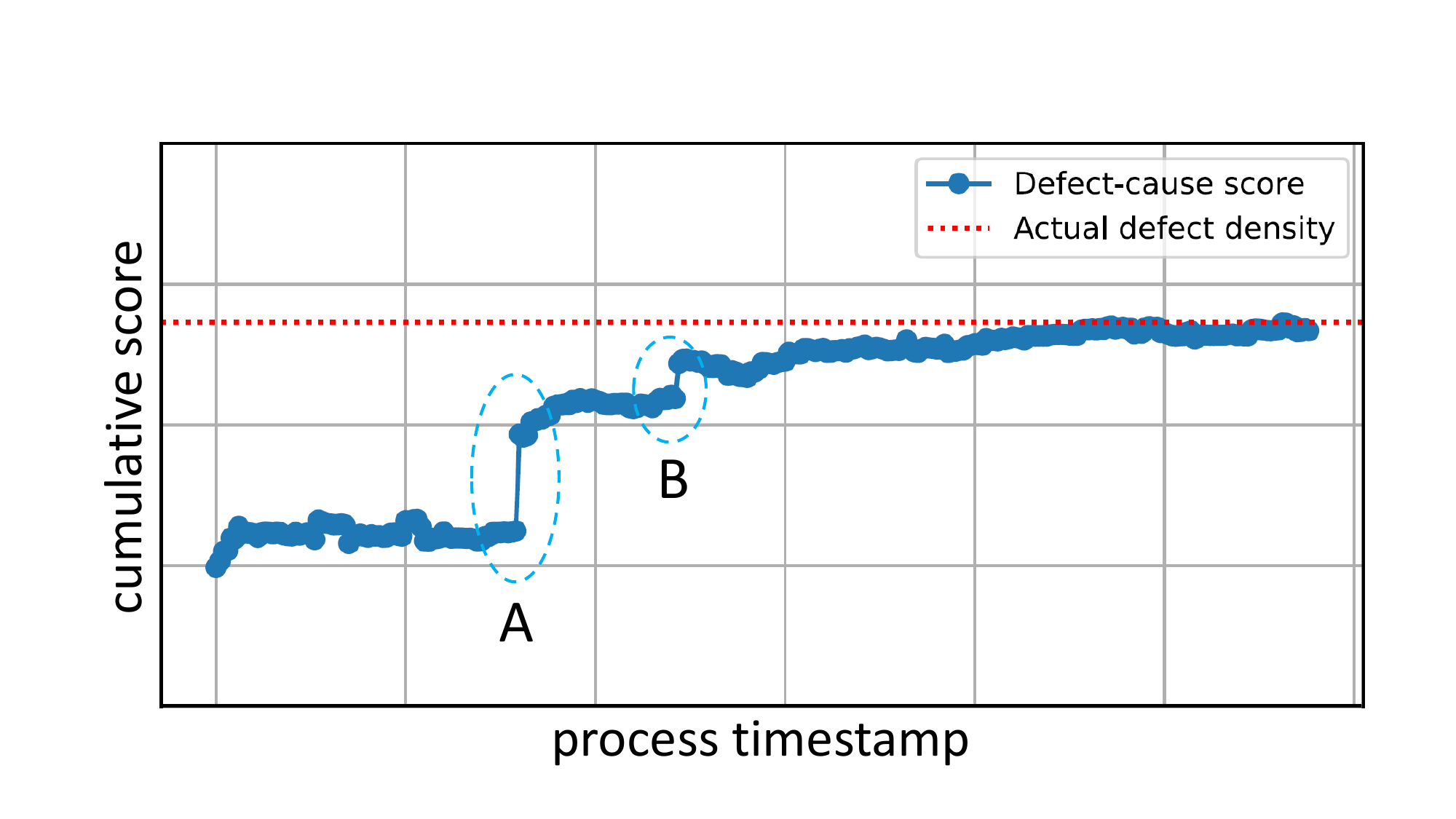}
    \caption{Cumulative attribution score (``defect-cause score'') for a specific wafer with high defect density, highlighting potential problematic processes at points A and B.}
    \label{fig:CT_scoring}
\end{figure}

\section{Conclusion}

We have proposed a novel data-driven framework for root cause analysis, termed Partial Trajectory Regression (PTR).

The PTR framework consists of three key modules: Embedding, Regression, and Attribution. In the embedding module, we introduced proc2vec, a novel process representation approach based on string-level similarity. For the regression module, we proposed a unique trajectory regression framework that enables prediction from partial process trajectories. In the attribution module, we demonstrated that the proposed attribution scoring model satisfies the additive property and allows interpretations based on Rubin's causal model. 

Empirical evaluation of PTR on a state-of-the-art FEOL process at the NY CREATES Albany NanoTech fab demonstrated its effectiveness in identifying potential sources of wafer defects, providing actionable insights for process optimization.

\section*{Acknowledgement}

The authors gratefully acknowledge the support of NY CREATES and the Albany NanoTech Complex for providing access to state-of-the-art fabrication and characterization resources. They also extend their gratitude to Dr.~Monirul Islam and Dr.~Ishtiaq Ahsan for providing the PLY data and their valuable support throughout the project.

\bibliographystyle{ieeetr}
\bibliography{references}

\end{document}